\documentclass[sigconf]{aamas} 

\usepackage{balance} 

\makeatletter
\gdef\@copyrightpermission{
  \begin{minipage}{0.2\columnwidth}
   \href{https://creativecommons.org/licenses/by/4.0/}{\includegraphics[width=0.90\textwidth]{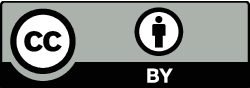}}
  \end{minipage}\hfill
  \begin{minipage}{0.8\columnwidth}
   \href{https://creativecommons.org/licenses/by/4.0/}{This work is licensed under a Creative Commons Attribution International 4.0 License.}
  \end{minipage}
  \vspace{5pt}
}
\makeatother

\setcopyright{ifaamas}
\acmConference[AAMAS '25]{Proc.\@ of the 24th International Conference
on Autonomous Agents and Multiagent Systems (AAMAS 2025)}{May 19 -- 23, 2025}
{Detroit, Michigan, USA}{Y.~Vorobeychik, S.~Das, A.~Nowé  (eds.)}
\copyrightyear{2025}
\acmYear{2025}
\acmDOI{}
\acmPrice{}
\acmISBN{}

\usepackage{amsthm}

\newtheorem{prop}{Proposition}
\theoremstyle{definition}
\newtheorem{deff}{Definition}
\newtheorem{remark}{Remark}

\title[AAMAS-2025 Manuscript]{Value Iteration for Learning Concurrently Executable Robotic Control Tasks}

\author{Sheikh A. Tahmid}
\affiliation{
  \institution{University of Waterloo}
  \city{Waterloo}
  \country{Canada}}
\email{sheikh.abrar.tahmid@uwaterloo.ca}

\author{Gennaro Notomista}
\affiliation{
  \institution{University of Waterloo}
  \city{Waterloo}
  \country{Canada}}
\email{gennaro.notomista@uwaterloo.ca}

\begin{abstract}
Many modern robotic systems such as multi-robot systems and manipulators exhibit redundancy, a property owing to which they are capable of executing multiple tasks. This work proposes a novel method, based on the Reinforcement Learning (RL) paradigm, to train redundant robots to be able to execute multiple tasks concurrently. Our approach differs from typical multi-objective RL methods insofar as the learned tasks can be combined and executed in possibly time-varying prioritized stacks. We do so by first defining a notion of task independence between learned value functions. We then use our definition of task independence to propose a cost functional that encourages a policy, based on an approximated value function, to accomplish its control objective while minimally interfering with the execution of higher priority tasks. This allows us to train a set of control policies that can be executed simultaneously. We also introduce a version of fitted value iteration to learn to approximate our proposed cost functional efficiently. We demonstrate our approach on several scenarios and robotic systems.
\end{abstract}


\newcommand{\BibTeX}{\rm B\kern-.05em{\sc i\kern-.025em b}\kern-.08em\TeX}

\begin{document}


\pagestyle{fancy}
\fancyhead{}


\maketitle 


\section{Introduction}
Training a robot to be able to concurrently execute multiple control objectives is an active area of research in robotics and Reinforcement Learning (RL).
This is especially useful for robotic systems that have the capability of executing multiple tasks concurrently --- also known as redundancy --- such as multi-robot systems \cite{redundancyBook}, \cite{multiRobotRedundancy1}, \cite{multiRobotRedundancy2}, manipulators \cite{redundancyBook}, \cite{robotManipulatorsRedundancy1}, \cite{robotManipulatorsRedundancy2} and humanoids \cite{humanoidsRedundancy1}, \cite{humanoidsRedundancy2}.
Multi-objective RL \cite{multiObjSurveyNew}, \cite{multiObjSurveyOld}, which aims to train an agent to solve multiple possibly competing objectives, is a possible approach that has been applied to robotic control problems \cite{multiObjRlRobotics1}, \cite{multiObjRlRobotics2}.
While multi-objective RL methods may work, they generally result in polices that cannot execute subsets of the objectives or reprioritize tasks during online execution.\par
In this paper, we propose an alternative approach where we instead train multiple tasks using RL to be compatible with each other.
This allows us to combine possibly time-varying subsets of tasks which may be useful in many systems such as those operating over long durations where objectives are likely to change over time \cite{egerstedt2021robot}.
Furthermore, several works that study concurrently executing multiple tasks using redundant robots focus on executing tasks in a prioritized fashion \cite{jacobianPriorities}, \cite{robotManipulatorsRedundancy1}, \cite{esb}, \cite{gennaroEsbConfPaper}.
This is useful in robotics applications where one should, for instance, prioritize safety-critical tasks over others.
Our approach in this paper extends this idea as our method encourages each task's learned policy to not interfere with the execution of higher priority tasks.
Several works propose methods for combining multiple deep RL policies together \cite{todorovNips}, \cite{composableDeepRLIcra}, \cite{mcpNips2019}, \cite{composeIcml}, \cite{gennaroDars}.
We opt to make use of the work in \cite{gennaroDars} which combines multiple value functions implemented as neural networks---which may be learned through various methods within RL \cite{actorCriticSurvey}, \cite{cfvi}---into a single min-norm controller capable of concurrently executing each task in a prioritized fashion. 
The controller in \cite{gennaroDars} allows us to execute possibly time-varying subsets of tasks with possibly time-varying priorities.
The work in \cite{gennaroDars}, however, does not enforce nor encourage any constraints on the learned functions themselves, causing there to be no guarantees on how trained tasks execute together. 
In our work, we propose a particular cost functional that can be used to encourage learned value functions to satisfy a notion of independence, defined based on definitions of independence in \cite{jacobianTasks} and \cite{esb}, which allows the tasks to be executed concurrently.\par
We believe our proposed approach can be especially useful in applications that involve the coordinated control of multi-robot systems. In persistent environmental monitoring with multiple robots \cite{Ma2018}, \cite{gennaroRlEnvMonit}, for example, a team of robots must efficiently explore and monitor the given environment while maintaining battery levels. Being able to switch priorities or the set of tasks being executed would also be useful in such an application in order to respond to emergencies or monitor particular areas of interest. Multi-robot systems are, in general, not differentially flat \cite{CortesAndEgerstedt}, resulting in it being difficult to plan their behaviour using classical methods. As a result, RL has shown to be a useful tool in controlling multi-robot systems to perform complex tasks \cite{gennaroRlEnvMonit}.\par
The main contributions of this work are as follows:
\begin{itemize}
	\item We define notions of task \textit{independence} and \textit{orthogonality} for learned cost-to-go functions that characterize the ability to concurrently execute certain tasks together.
	\item We propose a form of a cost functional and show with a proof that a set of tasks can be made \textit{independent}, and therefore possible to execute together concurrently, by solving our proposed cost functional.
	\item We propose a version of fitted value iteration \cite{bertsekas} for the continuous control setting, similar to \cite{cfvi}, to directly learn our proposed cost functional within the deep RL paradigm.
	\item We demonstrate our approach by training control policies for robotic systems in such a way that the policies' respective tasks, which may have otherwise had conflicting objectives, can be simultaneously executed.
\end{itemize}
\section{Related Works}
This paper is concerned with training robotic control tasks within the RL paradigm, with the overall aim being to execute multiple tasks concurrently, which is one shared by existing multi-objective RL methods.
We define notions of independence and orthogonality between learned tasks, similar to definitions from existing work discussed in 2.3. 
We combine learned tasks using a min-norm controller proposed in \cite{gennaroDars}, which unlike other similar works, allows us to execute learned tasks in possibly time-varying prioritized stacks.

\subsection{RL and Fitted Value Iteration for Robotic Control Tasks}
Existing work in deep RL for continuous control tasks in robotics predominantly uses policy iteration methods such as \cite{gae}, \cite{ddpg}.
When the system dynamics are known, one alternative is to train a neural network to approximate an optimal cost-to-go function through fitted value iteration \cite{bertsekas}, 
where each iteration has the following two main steps:
\begin{equation} \tag{1}
\begin{aligned}
	J_{\textrm{tar}}^k (x_0) &= \min_{u_1,\ldots,u_l \in \mathcal{U}} \sum_{i = 0}^{l - 1} \gamma^i \cdot c(x_i, u_i) + \gamma^l J( x_l; \theta_k ), \\ 
	& \forall x_0 \in \mathcal{D} \\
\end{aligned}
\end{equation}
\begin{equation} \tag{2}
\begin{aligned}
	\theta_{k + 1} &= \min_{\theta} \frac{1}{| \mathcal{D} |} \sum_{x \in \mathcal{D}} { \lVert J(x; \theta) - J_{\textrm{tar}}^k (x) \rVert }^2 \\
\end{aligned}
\end{equation}
with state space $\mathcal{X} \subseteq \mathbb{R}^n$, input space $\mathcal{U} \subseteq \mathbb{R}^m$, network parameters $\theta \in \mathbb{R}^d$, instantaneous cost function $c(x, u) : \mathbb{R}^n \times \mathbb{R}^m \rightarrow \mathbb{R}_{\ge 0}$, a sequence of states $x_1, \ldots, x_l \in \mathcal{X}$ reached by selecting the sequence of inputs $u_1, \ldots, u_n \in \mathcal{U}$, discount factor $0 < \gamma \le 1$, number of lookahead steps $l \ge 1$ and dataset of states $\mathcal{D}$. 
To avoid biased estimates based on a particular choice of the number of look-ahead steps $l$, we can use the forward view of TD($\lambda$) \cite{suttonBarto}.
We assume the system to have deterministic dynamics.
Note that RL is typically posed as a maximization problem where a value function may be learned, but in this paper, we pose it as a minimization problem where the value function may also be referred to as a cost-to-go function. We will use the two terms interchangeably in this paper.
For robotic control tasks, the state and input space is continuous, meaning that in the general case, it is not clear how to find the inputs $u_1, \ldots, u_n \in \mathcal{U} \subseteq \mathbb{R}^m$ that minimize (1).
The work in \cite{cfvi} proposes the Continuous Fitted Value Iteration (CFVI) algorithm.
The work in \cite{cfvi}, which extends work in \cite{doya}, shows that the optimization problem in (1) has an analytical solution for the continuous control setting under certain assumptions.
In our work, we propose a similar version of fitted value iteration as CFVI for a similar problem setting with different assumptions on the instantaneous cost.\par
\subsection{Multi-Objective RL}
Multi-objective RL has been used to solve decision and control problems that involve optimizing for multiple objectives \cite{multiObjSurveyNew}, \cite{multiObjSurveyOld}.
It is possible to treat multi-objective RL as regular RL by using a linear combination of instantaneous rewards.
However, this has problems discussed in \cite{multiObjSurveyNew} such as the increased complexity of reward-shaping and the lack of explainability in resulting policies.
The work in \cite{multiObjRlRobotics1} and \cite{multiObjRlRobotics2}, which both find Pareto optimal policies, show that multi-objective RL is viable in robotic control problems.
However, multi-objective RL methods generally result in policies that are not separable nor capable of complying with time-varying task priorities, unlike our
proposed approach of training multiple tasks to be compatible with each other.\par
\subsection{Definitions of Control Tasks and Task Independence}
There exist previous work that mathematically define robotic control tasks, along with notions of \textit{independence} and \textit{orthogonality} which affect the ability to concurrently execute multiple tasks.
Several works such as \cite{oldJacobianBasedTasks}, \cite{jacobianPriorities} have studied executing multiple Jacobian-based tasks concurrently, often in a prioritized fashion, for redundant kinematic robots.
The work in \cite{jacobianTasks} analyzes the stability of such algorithms, also defining notions of independence and orthogonality between Jacobian-based tasks.
Set-based tasks were introduced in \cite{setBasedTasksOg} and the concurrent execution of multiple of them is analyzed in \cite{setBasedTasks}.
The work in \cite{esb} uses Extended-Set Based tasks \cite{gennaroEsbConfPaper} to generalize Jacobian-based tasks, as well as the definitions of independence and orthogonality introduced in \cite{jacobianTasks}.
The definitions in \cite{esb} can be used to describe the cost-to-go functions that we train in this paper.
However, these definitions do not take into account the system dynamics, unlike Definition $1$ in Section \uppercase\expandafter{\romannumeral3\relax} of this work which does.

\subsection{Composition of Learned Control Tasks}
Section 2.1 in \cite{gennaroDars} summarizes some of the previous literature on composing learned deep RL policies \cite{todorovNips}, \cite{composableDeepRLIcra}, \cite{mcpNips2019}, \cite{composeIcml}.
In our work, we opt to use the min-norm controller proposed in \cite{gennaroDars} as our method of combining multiple learned tasks
as it allows us to execute possibly time-varying subsets of tasks with possibly time-varying priorities.
We explain the controller in \cite{gennaroDars} further in the next section.
While \cite{gennaroDars} proposes a method of combining multiple learned tasks, it does not offer a way to ensure that the learned tasks are compatible with each other.
Our proposed method allows us to train tasks in such a way that it is possible to combine and execute them together using the method in \cite{gennaroDars}.

\section{Background and Problem Formulation}
We consider the problem of using RL to learn a set of tasks for a redundant robotic system in such a way that the tasks can be executed concurrently.\par
Firstly, we assume that the system is deterministic and control-affine.
These assumptions may not be applicable to all applications.
However, in robotics, it is common to assume that the dynamics are approximately known and many robotic systems, including most mobile robots and manipulators, can be modelled as control-affine systems.
The dynamics can be represented with: 
$$
\dot{x}(t) = f(x) + g(x) u \eqno{(3)}
$$
where $x \in \mathcal{X} \subseteq \mathbb{R}^n$ and $u \in \mathcal{U} \subseteq \mathbb{R}^m$ denote the state and input, respectively, and 
$f : \mathbb{R}^n \rightarrow \mathbb{R}^n$ and $g : \mathbb{R}^n \rightarrow \mathbb{R}^{n \times m}$ are continuous vector fields.\par
We assume that each task $1, \ldots, N$ is learned sequentially in the order of importance with task $1$ being the most important. 
This is a reasonable assumption in robotics as it is often desirable to prioritize more safety-critical tasks over others.
We assume that each task $i$ has an associated state cost, $q_i(x)$, that is positive semi-definite.
To learn each task $i$, we consider solving the following infinite horizon problem:
$$
J^{\star}_i(x(t)) = \min_{u(\cdot)} \int_t^{\infty} e^{- \beta \tau } \left( q_i(x) + u^{\top} R_i(x) u \right) d \tau \eqno{(4)}
$$
where $\beta \ge 0$ and $R_i : \mathbb{R}^n \rightarrow \mathbb{S}^{n}_{++} $ maps each state $x$ to a positive definite matrix. 
We assume that each task $i$ has a reachable terminal goal state $x_T \in \mathcal{X}$ where $q_i(x_T) = 0$.
As this problem cannot be solved analytically in the general case, we consider using deep RL to learn parametric approximations which we will denote as $\tilde{J}_i$ for each task $i$.
Note that the continuous dynamics of the system described in (3) can be discretized which would allow us to use an actor-critic \cite{actorCriticSurvey} method 
or version of fitted value iteration \cite{bertsekas}, described in (1) and (2), to find some $\tilde{J}_i$.\par
If we learn an approximation $\tilde{J}_i$ for task $i$, one way to execute the optimal policy is to use as input the $u^{\star}(.)$ that solves the optimization problem in (4). 
However, to simply make progress on task $i$, we can pick an input that drives $\tilde{J}_i$ toward $0$ or, in other words, makes the time derivative negative at a particular state.
To make progress on task $i$ at every state $x \in \mathcal{X}$ as fast as if we were executing the optimal policy, we can pick an input $u \in \mathcal{U}$ that makes the time derivative negative by the same amount as if we were explicitly using the optimal input $u^{\star}(x)$ at state $x$.
This motivates the use of the pointwise min-norm controller proposed in \cite{gennaroDars}.\par
Given a set of learned cost-to-go functions $\tilde{J}_1,\ldots,\tilde{J}_N$, each implemented with a neural network, we can select inputs that execute each task concurrently using the min-norm quadratic program proposed in \cite{gennaroDars}:
\begin{equation} \tag{5}
\begin{aligned}
	\min_{u \in \mathcal{U}, \delta \in \mathbb{R}^N} \quad & { \lVert u \rVert }^2 + \kappa { \lVert \delta \rVert }^2 \\
	\textrm{s.t.} \quad & L_f \tilde{J}_1 (x) + L_g \tilde{J}_1 (x) u \le - \sigma_1(x) + \delta_1 \\
	\quad & \vdots \\
	\quad & L_f \tilde{J}_N (x) + L_g \tilde{J}_N (x) u \le - \sigma_N(x) + \delta_N \\
	\quad & K \delta \ge 0
\end{aligned}
\end{equation}
Note that $L_f \tilde{J}_i (x)$ and $L_g \tilde{J}_i (x)$ are Lie derivatives where $L_f \tilde{J}_i (x) = \frac{ \partial \tilde{J}_i }{ \partial x} (x) f(x)$ and $L_g \tilde{J}_i (x) = \frac{ \partial \tilde{J}_i }{ \partial x} (x) g(x)$, meaning that the lefthand side of each constraint is simply the time derivative of the cost-to-go function for a task.
Since each $J^{\star}_1,\ldots,J^{\star}_N$ is positive semi-definite and smaller values represent being closer to a goal state, each corresponding $\tilde{J}_1,\ldots,\tilde{J}_N$ can be treated as a candidate
Control Lyapunov Function (CLF). Note that since $\tilde{J}_1,\ldots,\tilde{J}_N$ are approximated by neural networks, $L_f \tilde{J}_i (x)$ and $L_g \tilde{J}_i (x)$ can be calculated using back-propagation.
As mentioned previously, picking an input $u$ that makes the time derivative of a particular $\tilde{J}_i$ negative represents making progress on that task.\par
$\sigma_1(x), \ldots, \sigma_N(x)$ are positive semi-definite scalar functions, parameterized by the current state $x$, designed to recover inputs equivalent to the optimal input with respect to each approximate cost-to-go function.
\begin{remark}
Note that the expressions $\sigma_1(x), \ldots, \sigma_N(x)$ in \cite{gennaroDars} are derived from \cite{clfPrimbs1}, \cite{clfPrimbs2} assuming that the approximated cost functionals do not use a discount factor. 
However, it is still reasonable to use the controller in (5) with tasks trained with a slight discount factor, meaning $\gamma$ close to $1$ in the discrete case in (1) and $\beta$ close to $0$ in the equivalent continuous formulation in (4), as the recovered policies would still be close enough to optimal.
\end{remark}\par
Slack variables $\delta_1, \ldots, \delta_N \in \mathbb{R}$, which form $\delta \in \mathbb{R}^N$, are added to (5) to make the quadratic program feasible for every state $x \in \mathcal{X}$ as well as prioritize the tasks.
$K \in \mathbb{R}^{N \times N}$ is used to specify the possibly time-varying priorities between tasks by specifying the relationship between different values of $\delta_i$ for $i \in \{1, \ldots, N\}$.
Higher priority tasks are allowed less slack than lower priority ones, meaning that the controller will prioritize picking inputs close to optimal for higher priority tasks than lower priority ones.
$\kappa \ge 0$ is a hyperparamter used in the objective function in (5) to control the total amount of slack allowed.\par
If the quadratic program in (5) is used to execute our learned tasks, being able to execute $N$ tasks concurrently then depends on the relationships between  $L_g \tilde{J}_1 (x), \ldots, L_g \tilde{J}_N (x) $.
Specifically, if $(L_g \tilde{J}_1 (x))^{\top}, \ldots, (L_g \tilde{J}_N (x))^{\top} $ are linearly independent at some state $x \in \mathcal{X}$, then it is possible to pick an input $u \in \mathcal{U}$ at that state that drives each $\tilde{J}_1, \ldots, \tilde{J}_N$ toward $0$ and therefore make progress on each task, assuming that there are no additional constraints on the input.
\begin{deff}[Independent and Orthogonal Tasks]
	Consider a set of differentiable, positive semi-definite functions encoding tasks for a control affine system (1), of the form $V_i : \mathcal{X} \rightarrow \mathbb{R}_{\ge 0}$.
	We assume that at least one of $L_g V_1, \ldots, L_g V_N$ is non-zero $\forall x \in \mathcal{X}$.
	We also assume that $\forall i \in \{1, \ldots, N\}$, $L_g V_i (x) = 0$ only if $V_i(x) = 0$.
	$V_1, \ldots, V_N$ are \textit{independent} to each other at state $x$ if the nonzero vectors within $(L_g V_1 (x) )^{\top}, \ldots, (L_g V_N (x) )^{\top}$ are linearly independent.
	$V_1, \ldots, V_N$ are \textit{orthogonal} to each other at state $x$ if\newline
	$\langle (L_g V_i (x) )^{\top}, (L_g V_j (x) )^{\top} \rangle = 0 \, \forall i, j \in \{1, \ldots, N \}$.
	Note that if tasks are orthogonal at $x$, then they are also independent at $x$.
\end{deff}
\begin{remark}
	Assuming that at least one of $L_g V_1, \ldots, L_g V_N $ is non-zero is a reasonable assumption in our application as it is unlikely for the gradients of each of the trained neural networks to be perfectly zero.
\end{remark}\par
If a set of tasks are independent at every state $x \in \mathcal{X}$, then there is always a feasible solution $u \in \mathcal{U}$ to the problem in (5) that makes the time derivative of each $\tilde{J}_1,\ldots,\tilde{J}_N$
negative when the respective task has not already been completed, thus driving each $\tilde{J}_1, \ldots, \tilde{J}_N$ toward $0$, and therefore concurrently executing each task. 
In our chosen setting, the problem of training a robotic system to be able to execute multiple objectives deduces to training a set of learned tasks $\tilde{J}_1, \ldots, \tilde{J}_N$ such that
$\tilde{J}_1, \ldots, \tilde{J}_N$ are independent at every state $x \in \mathcal{X}$. 
In this paper, we solve this problem by designing a cost functional, that when approximated using an RL algorithm, results in learned tasks that satisfy our definition of independence.
This, along with a variant of the fitted value iteration algorithm to fit this proposed cost functional, are two of the main contributions of the paper and are the subject of the next section.

\section{Main Results}
In 4.1, we propose a version of the fitted value iteration algorithm for approximating cost functionals of the form in (4) for the continuous control setting.
We can use this algorithm to approximate a cost functional that we propose in 4.2. In 4.2, we prove, among other things, that our proposed cost functional can be used to define learned tasks to be independent to each other and thus possible to execute simultaneously.
\subsection{Continuous Fitted Value Iteration}
We propose a form of fitted value iteration for learning cost-to-go functions of the form in (4).
The main idea of the CFVI \cite{cfvi} algorithm is to extend the fitted value iteration algorithm by solving the optimization problem in (1) analytically for the continuous control setting. 
We extend this approach in Proposition 1 for cost functionals of the form in (4).
\par
\begin{prop} Given a control-affine system and the $k$-th iteration of an estimate of the cost-to-go function described in (4), $J^k$, the optimal policy
	based on the fitted value iteration update step in (1) in the continuous setting is:
	\begin{align*}
		u^{\star}(t) = - \frac{1}{2} {R(x)}^{-1} \left( L_g J^k (x) \right)^{\top} 
	\end{align*}
\end{prop}
\begin{proof}
	This proof is very similar to Theorem $1$ in \cite{cfvi}. 
	We first define discretized dynamics for the system by using a $\Delta t$ that is small enough to appproximate the dynamics in (3).
	We define $\bar{f}(x, u)$ such that:
	$$
		\bar{f}(x, u) = x + (f(x) + g(x) u) \Delta t
	$$
	where $\bar{f}(x,u)$ calculates the new state after sending input $u \in \mathcal{U}$ at state $x \in \mathcal{X}$ for a duration of $\Delta t$.
	We then substitute $\bar{f}(x, u)$ and the instantaneous cost into the fitted value iteration update step in (1) with $l = 1$:
	$$
	J_{\textrm{tar}}^k (x) = \min_{u \in \mathcal{U}} \left\{ \left( q(x) + u^{\top} R(x) u \right) \cdot \Delta t + \gamma J^k(\bar{f}(x, u)) \right\}.
	$$
	We can then rewrite this using the Taylor approximation of $J^k$ around $x$: 
	\begin{align*}
	J_{\textrm{tar}}^k (x) &= \min_{u \in \mathcal{U}} \left\{ \left( q(x) + u^{\top} R(x) u \right) \Delta t \right. \\
	& \left. + \, \gamma \left( J^k(x) + \left( \frac{\partial J^k}{\partial x} \right)^{\top} \bar{f}(x, u) \Delta t + \mathcal{O}(\Delta t, x, u) \right) \right\}
	\end{align*}
	where $\mathcal{O}(\Delta t, x, u)$ represents the higher order terms.
	In the continuous setting, $\Delta t$ approaches $0$, which means that the higher order terms, $\mathcal{O}(\Delta t, x, u)$, also approach $0$. 
	By the same logic in \cite{cfvi}, $\gamma$ also approaches $1$ at the continuous time limit as $\Delta t$ approaches $0$.
	To find the optimal input, we can then simplify, removing terms that do not rely on $u$, removing $\mathcal{O}(\Delta t, x, u)$ and substituting $\gamma = 1$ to get the following unconstrained optimization problem:
	\begin{align*}
		\min \left\{ u^{\top} R(x) u + L_g J^k (x) u \right\}
	\end{align*}
	for which the solution is $ u^{\star} = - \frac{1}{2} { R(x) }^{-1} \left( L_g J^k (x) \right)^{\top} $. Note that ${R(x)}^{-1}$ is defined since $R(x)$ is positive definite.
\end{proof}\par
We can then use this analytical solution for the optimal policy with respect to $J^k$ to perform the fitted value iteration steps shown in (1) and (2). 
To avoid biased estimates for a particular choice of $l$, we use the forward view of TD($\lambda$) \cite{suttonBarto}.\par
This version of fitted value iteration is very similar to CFVI \cite{cfvi}. 
However, instead of assuming that the input cost is separable from the state and strictly convex, we instead assume that it is quadratic with respect to some function that maps the current state to a positive definite matrix. 
This differing assumption allows us to propose cost functionals, of a form different than those that the CFVI algorithm is able to learn, for training tasks that satisfy our definition of independence.
\subsection{Cost Functionals for Independence}
Consider the following cost functional:
$$
\min_{u(\cdot)} \int_t^{\infty} e^{- \beta \tau } \left( q_{N + 1}(x) + { \lVert u \rVert }^2 + \sum_{i = 1}^N ( L_g \tilde{J}_i (x) u)^2 \lambda_i \right) d \tau \eqno{(6)}
$$
\noindent where $\lambda_1,\ldots,\lambda_N > 0$, $\beta \ge 0$ and $\tilde{J}_1, \ldots, \tilde{J}_N$ are a set of $N$ learned tasks that we wish to train a new task to be independent to. 
It is clear that (6) is an instance of (4), as ${ \lVert u \rVert }^2 + \sum_{i = 1}^N ( L_g V_i (x) u)^2 \lambda_i$
can be rearranged as 
$$
u^{\top} \left( I + \sum_{i = 1}^N ( L_g \tilde{J}_i (x) )^{\top} ( L_g \tilde{J}_i (x) ) \lambda_i \right) u
$$
and $I + \sum_{i = 1}^N ( L_g \tilde{J}_i (x) )^{\top} ( L_g \tilde{J}_i (x) ) \lambda_i$
is positive definite. This means that the method of value iteration from the previous subsection can be used to approximate the cost functional proposed in (6). 
$\lambda_1,\ldots,\lambda_N$ are hyperparameters that determine how much to penalize inputs which interfere with previously trained, higher priority tasks.
Through Proposition 2, we show that choosing a large enough $\lambda_1, \ldots, \lambda_N$ would result in a newly trained task, that fits our proposed cost functional, to be independent to $\tilde{J}_1, \ldots \tilde{J}_N$.\par
\begin{prop}
	There exist scalars $\lambda_1, \ldots, \lambda_N > 0$, such that if $\tilde{J}_{N + 1}$ fits a cost functional of the form in (6) that uses $\lambda_1, \ldots, \lambda_N$, 
	then for every state $x$ where $\tilde{J}_1, \ldots, \tilde{J}_N$ are independent, $\tilde{J}_1, \ldots, \tilde{J}_N, \tilde{J}_{N + 1}$ are also independent at $x$.
\end{prop}
\begin{proof}
	If our deep RL algorithm for training the function $\tilde{J}_{N + 1}$ converges such that it solves the HJB equation for the cost functional in (6), then the optimal input, based on
	Proposition 1 as well as the work in \cite{doya}, minimizes the following unconstrained optimization problem:
	\begin{align*}
		\min_{u} \left\{ { \lVert u \rVert }^2 + \sum_{i = 1}^N (L_g \tilde{J}_i (x) u)^2 \lambda_i + L_g \tilde{J}_{N + 1} (x) u \right\}
	\end{align*}
	The solution to this is $u^{\star} = - \frac{1}{2} {R(x)}^{-1} ( L_g \tilde{J}_{N + 1} (x) )^{\top} $ where $R(x) = I + \sum_{i = 1}^N ( L_g \tilde{J}_i (x) )^{\top} L_g \tilde{J}_i (x) $.
	If $\lambda_1, \ldots, \lambda_N$ are large enough, $u^{\star}$ approximates a solution to the following
	constrained optimization problem:
		\begin{align*}
			&\min_{u} \left\{ { \lVert u \rVert }^2 + L_g \tilde{J}_{N + 1} (x) u \right\} \\
			&\text{s.t.} \sum_{i = 1}^N ( L_g \tilde{J}_i (x) u )^2 = 0
		\end{align*}
	which means that $\langle L_g \tilde{J}_i (x), - \frac{1}{2} { R(x) }^{-1} ( L_g \tilde{J}_{N + 1} (x) )^{\top} \rangle = 0$ for all $i \in \{ 1, \ldots, N \}$. 
	For the sake of a contradiction, assume that for some $x \in \mathcal{X}$, $\tilde{J}_1,\ldots,\tilde{J}_N$ are independent at $x$ and 
	$\tilde{J}_1,\ldots,\tilde{J}_N, \tilde{J}_{N + 1}$ are not independent at $x$ meaning that $L_g \tilde{J}_{N + 1} (x) $ can be written as a linear combination of 
	$L_g \tilde{J}_1 (x), \ldots, L_g \tilde{J}_N (x)$.
	Let $c_1, \ldots, c_N \in \mathbb{R}$ be the multipliers for this linear combination.
	We can see that 
	$$\langle c_i (L_g \tilde{J}_i (x) )^{\top}, - \frac{1}{2} { R(x) }^{-1} (L_g \tilde{J}_{N + 1} (x) )^{\top} \rangle = 0 \,, 1 \le i \le N.$$
	By summing each inner product together, we can then see that $ \langle ( L_g \tilde{J}_{N + 1} (x) )^{\top}, {R(x)}^{-1} ( L_g \tilde{J}_{N + 1} (x) )^{\top} \rangle = 0 $.
	$R(x)$ is positive definite, so its inverse must also be positive definite. This means that $(L_g \tilde{J}_{N + 1} (x) )^{\top}$ must be $0$.
	However, we assumed $\tilde{J}_1, \ldots, \tilde{J}_N$ to be independent, so a linear combination of $(L_g \tilde{J}_1 (x) )^{\top}, \ldots, (L_g \tilde{J}_N (x) )^{\top}$ cannot be $0$, so this
	is a contradiction.  
	Therefore, if $\tilde{J}_1, \ldots, \tilde{J}_N$ are independent at $x$, then $\tilde{J}_1, \ldots, \tilde{J}_N, \tilde{J}_{N + 1}$ are also independent at $x$.
\end{proof}\par

Proposition 2 suggests that we can train a set of $N$ tasks to be independent by training them in order using the cost functional in (6) with large enough choices of $\lambda_1, \ldots, \lambda_N$.\par
In some cases, it is possible for a resulting learned task $\tilde{J}_{N + 1}$ to not only be independent to previously trained tasks
$\tilde{J}_1, \ldots, \tilde{J}_N$, but to also be orthogonal. We show in Proposition 3 that if $\tilde{J}_1, \ldots, \tilde{J}_N$ are independent, then the optimal policy to execute task $\tilde{J}_{N + 1}$ is $- \frac{1}{2} (L_g \tilde{J}_{N +1} (x) )^{\top}$ if and only if $\tilde{J}_{N + 1}$ is orthogonal to each of $\tilde{J}_1, \ldots, \tilde{J}_N$.\par

\begin{prop}
	Assume that a function $\tilde{J}_{N + 1}$ solves the HJB equation and is defined based on the cost functional in (6). Without loss of generality, let $\lambda_1 = \cdots = \lambda_N = 1$.
	We assume as before that at least one of
	$L_g \tilde{J}_1(x), \ldots, L_g \tilde{J}_N(x), L_g \tilde{J}_{N + 1}(x)$ is non-zero $\forall x \in \mathcal{X}$.
	At a state $x$, if $J_1, \ldots, J_N$ in the cost functional of $\tilde{J}_{N + 1}$ are independent, 
	then $\tilde{J}_{N + 1}$ is orthogonal to each of $J_i \in \{J_1, \ldots, J_N \}$ at $x$ if and only if
	the optimal input at $x$ with respect to $\tilde{J}_{N + 1}$ is equal to $- \frac{1}{2} ( L_g \tilde{J}_{N + 1} (x) )^{\top}$. 
\end{prop}
\begin{proof}
	\begin{enumerate}
	\item[($\Rightarrow$)]
		Assume that $\tilde{J}_{N + 1}$ is orthogonal to each of $\tilde{J}_1, \ldots, \tilde{J}_N$ at $x$. As shown in the work in \cite{doya} and used similarly in Proposition 2, the optimal input with respect to $\tilde{J}_{N + 1}$ is an input that
		minimizes the following optimization problem:
		\begin{align*}
			\min_{u} \left\{ { \lVert u \rVert }^2 + \sum_{i = 1}^N ( L_g \tilde{J}_i (x) u )^2 +  L_g \tilde{J}_{N + 1} (x) u \right\}
		\end{align*}
		We can see that if $\tilde{J}_{N + 1}$ is orthogonal to each of $\tilde{J}_1, \ldots, \tilde{J}_N$ at $x$, then 
		\begin{align*}
			\sum_{i = 1}^N \left( L_g \tilde{J}_i (x) \left( - \frac{1}{2} ( L_g \tilde{J}_{N + 1} )^{\top} (x) \right) \right)^2 &= 0
		\end{align*}
			We also know that $- \frac{1}{2} ( L_g \tilde{J}_{N + 1} (x) )^{\top} $ minimizes the function ${ \lVert u \rVert }^2 + L_g \tilde{J}_{N + 1} (x) u$. Therefore, if $\tilde{J}_{N + 1}$ is orthogonal to each of $\tilde{J}_1, \ldots, \tilde{J}_N$ at $x$, then the optimal input at $x$ is $- \frac{1}{2} ( L_g \tilde{J}_{N + 1} (x) )^{\top}$. 
		\item[($\Leftarrow$)] Assume that $u^{\star} = - \frac{1}{2} ( L_g \tilde{J}_{N + 1} (x) )^{\top} $ is the optimal input. We know that the optimal input must solve the following optimization problem:
		\begin{align*}
			\min_{u} \left\{ { \lVert u \rVert }^2 + \sum_{i = 1}^N ( L_g \tilde{J}_i (x) u )^2 + L_g \tilde{J}_{N + 1} (x) u \right\}
		\end{align*}
		which means it must also solve the following equation when we take the gradient and set it equal to $0$:
		\begin{align*}
			2 u + 2 g(x)^{\top} \left( \sum_{i = 1}^N \left( \frac{\partial \tilde{J}_i}{\partial x} \right)^{\top} \frac{\partial \tilde{J}_i}{\partial x} \right) g(x) u + (L_g \tilde{J}_{N + 1} (x))^{\top} &= 0
		\end{align*}
		Substituting in $u^{\star} = - \frac{1}{2} ( L_g \tilde{J}_{N + 1} (x) )^{\top}$, we get:
		\begin{align*}
			2 g(x)^{\top} \left( \sum_{i = 1}^N \left( \frac{\partial \tilde{J}_i}{\partial x} \right)^{\top} \frac{\partial \tilde{J}_i}{\partial x} \right) g(x) u &= 0 \\
			\sum_{i = 1}^N (L_g \tilde{J}_i (x))^{\top} \langle ( L_g \tilde{J}_i (x) )^{\top}, ( L_g \tilde{J}_{N + 1} (x) )^{\top} \rangle &= 0
		\end{align*}
			Since we assumed that each of $\tilde{J}_1, \ldots, \tilde{J}_N$ are independent at $x$, we cannot find a set of non-zero constants that make this sum to $0$, so it must be the case that\newline $\langle (L_g \tilde{J}_i (x) )^{\top}, (L_g \tilde{J}_{N + 1} (x) )^{\top} \rangle = 0$ for each $i \in \{1, \ldots, N\}$. 
	\end{enumerate}
\end{proof}
If a task, $\tilde{J}_{N + 1}$, can be trained such that it is orthogonal to a set of other independent tasks, $\tilde{J}_1, \ldots ,\tilde{J}_N$, then not only is it possible to execute each task concurrently, it is also possible to execute the optimal input for $\tilde{J}_{N + 1}$ without having to calculate the gradients for tasks that $\tilde{J}_{N + 1}$ was trained to be independent to. If $\tilde{J}_{N + 1}$ was trained without a discount factor, then the min-norm controller in \cite{gennaroDars} would recover the exact optimal input for $\tilde{J}_{N + 1}$ without interfering with the execution of higher priority tasks.\par
	To summarize this section, we propose a cost functional in (6) that we prove in Proposition 2 can help us define tasks that are independent, and thus possible to execute together simultaneously. With Proposition 1, we propose a version of fitted value iteration \cite{bertsekas} that extends the continuous fitted value iteration algorithm \cite{cfvi} to fit this cost functional efficiently. Our method of fitted value iteration, similarly to \cite{cfvi}, can be accelerated vastly using GPUs, as calculating the optimal inputs with respect to the gradients, advancing the states to generate new value function estimates and updating the parameters of the neural network can be done in large batches using GPUs. Additionally, we show in Proposition 3 the relationship between the optimal input of a learned task and whether or not it is orthogonal to the set of tasks it was trained to be independent to.  \par
	In the next section, we demonstrate our proposed fitted value iteration algorithm and implications of Proposition $2$ in several scenarios, training robotic systems to be able to concurrently execute tasks that may have otherwise had conflicting objectives.



\section{Simulations}
In this section, we demonstrate our proposed method on several scenarios involving mobile robots. 
We show that we are able to use our approach to train cost-to-go functions, that may have had otherwise conflicting objectives, such that they are possible to combine and execute together concurrently. 
We compare our approach against simply learning the separate tasks using CFVI \cite{cfvi}.
\subsection{Go-to-Point and Avoid Square Region}
We first demonstrate our proposed method by training a mobile robot to move to a point while avoiding a square-shaped region.
The state space $\mathcal{X} = \mathbb{R}^2$ is the robot position and the input space $\mathcal{U} = \mathbb{R}^2$ is the velocity.\par
We first train the task of avoiding the region, using an instantaneous cost of $q_1(x) + { \lVert u \rVert }^2$ where $q_1(x) = 60$ when the robot is within the region and $0$ otherwise. 
We use CFVI to train a cost-to-go function, $\tilde{J}_1$, plotted in Fig. 1a.\par
Next, we train two versions of the task of moving the robot to the point ${\begin{bmatrix} -2 & 0 \end{bmatrix}}^{\top}$. 
	The first version is trained using CFVI with an instantaneous cost of $q_2(x) + { \lVert u \rVert }^2$, while the second version is trained, using our proposed method, with an instantaneous cost of $q_2(x) + { \lVert u \rVert }^2 + 10^4 ( L_g \tilde{J}_1 (x) u )^2$, where $q_2(x)$ is the distance from the current position $x \in \mathcal{X}$ to ${\begin{bmatrix} -2 & 0 \end{bmatrix}}^{\top}$ multiplied by $5$.
We will refer to the resulting functions, both plotted in Fig. 1, as $\tilde{J}_{\textrm{base}}$ and $\tilde{J}_{\textrm{ind}}$ respectively. 
We can see that there is an imprint of the region in the case of $\tilde{J}_{\textrm{ind}}$.\par
Finally, we combined $\tilde{J}_1$ with $\tilde{J}_{\textrm{base}}$ and $\tilde{J}_{\textrm{ind}}$ respectively using the min-norm controller in (5). The trajectories of the resulting controllers are shown in Fig. 2.
\begin{figure}[]
	\centering
	\begin{minipage}[c]{0.3\linewidth}
		\centering
		\includegraphics[width=\linewidth]{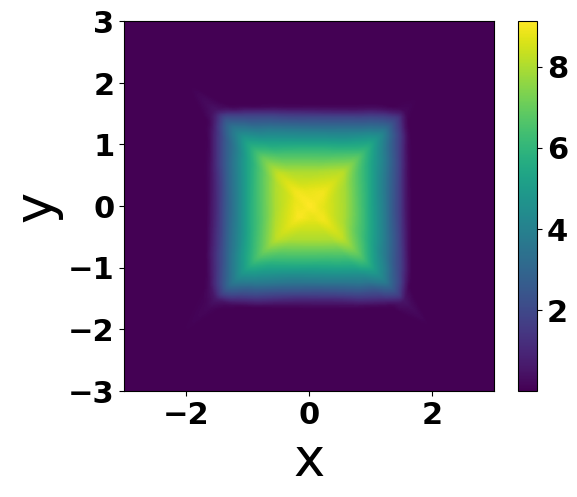}
		(a)
	\end{minipage}
	\begin{minipage}[c]{0.3\linewidth}
		\centering
		\includegraphics[width=\linewidth]{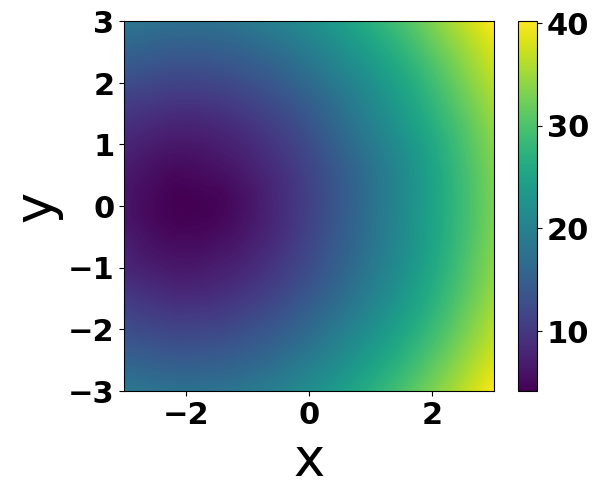}
		(b)
	\end{minipage}
	\begin{minipage}[c]{0.3\linewidth}
		\centering
		\includegraphics[width=\linewidth]{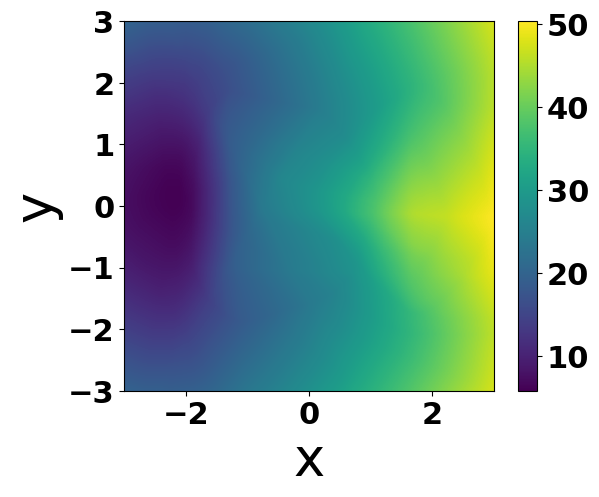}
		(c)
	\end{minipage}%
	\caption{(a): Heatmap of function, $\tilde{J}_1$, trained for avoidance task.
		 (b): Heatmap of function, $\tilde{J}_{\textrm{base}}$, trained using CFVI.
		 (c): Heatmap of function, $\tilde{J}_{\textrm{ind}}$, trained using proposed method to be independent to avoidance task.
		 Note that gradients of $\tilde{J}_{\textrm{ind}}$ appear to be linearly independent to those of $\tilde{J}_1$ as values in (c) appear to warp around the square region.
		 }
	\Description{
		Three plots containing heatmaps of the learned value functions discussed in this subsection. The first heat map is for the avoidance task. The second heatmap is for the go-to-point task
		trained using CFVI and the third heat map is the go-to-point task trained to be independent to the avoidance task using our proposed method. The values of the third heat map appear to 
		warp around the region that the robot must avoid.}

\end{figure}
We can see that since the gradients of $\tilde{J}_1$ and $\tilde{J}_{\textrm{base}}$ go in opposing directions, the agent is not able to move to the goal point from certain states. Meanwhile, our proposed approach influenced the gradients of
$\tilde{J}_{\textrm{ind}}$ and $\tilde{J}_1$ to be linearly independent, such that when they are combined with the min-norm controller in (5), the resulting controller moves the robot to the goal point while successfully avoiding the region.\par
This illustrative example demonstrates how our approach can very visibly influence the gradients of a learned cost-to-go function to be independent to other cost-to-go functions, resulting in tasks that can be executed concurrently.
\begin{figure}[]
	\centering
	\begin{minipage}[c]{0.45\linewidth}
		\centering
		\includegraphics[width=\linewidth]{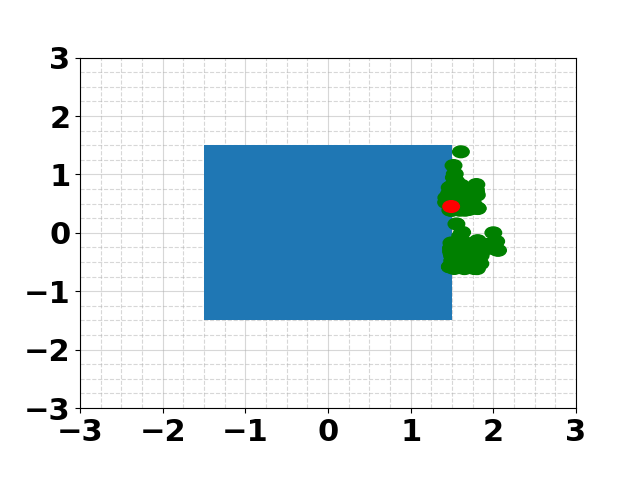}
		\footnotesize{(a) Combining $\tilde{J}_1$ with $\tilde{J}_{\textrm{base}}$}%
	\end{minipage}
	\begin{minipage}[c]{0.45\linewidth}
		\centering
		\includegraphics[width=\linewidth]{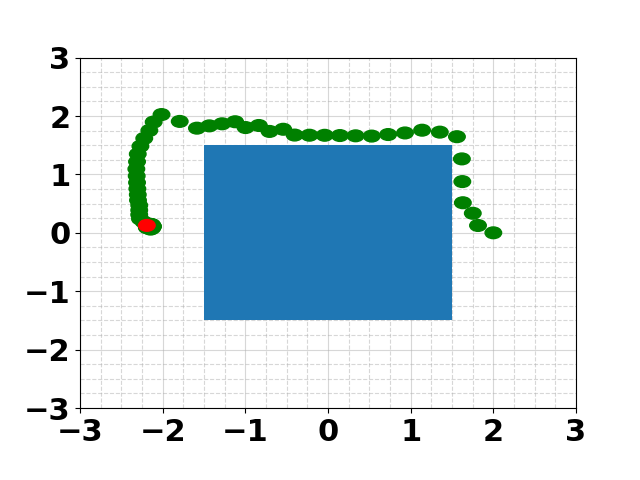}
		\footnotesize{(b) Combining $\tilde{J}_1$ with $\tilde{J}_{\textrm{ind}}$}%
	\end{minipage}%
	\caption{
		(a): Trajectory when combining avoidance task, $\tilde{J}_1$, with baseline go-to-point task, $\tilde{J}_{\textrm{base}}$.
		(b): Trajectory when combining $\tilde{J}_1$ with go-to-point task trained using proposed approach, $\tilde{J}_{\textrm{ind}}$.
		Green: Positions at each time step.
		Red: Final position for each trajectory.
		Blue: Region to avoid.
		}
	\Description{
		Trajectories of both go-to-point tasks combined with the avoidance task. The first trajectory, consisting of the baseline go-to-point task and the avoidance task has the robot getting
		stuck as the two tasks have conflicting objectives. The trajectory that uses the go-to-point method trained with our method actually reaches the goal.
		}
\end{figure}

\subsection{Form Shape and Avoid Square Region}
In this next example, we train a team of three mobile robots to avoid a square-shaped region in the centre of the environment and form a triangle simultaneously.
The state space, $\mathcal{X} = \mathbb{R}^6$, is the position of all robots where 
$\forall x \in \mathcal{X}, x = {\begin{bmatrix} p_1 & p_2 & p_3 \end{bmatrix}}^{\top}$
and $p_1, p_2, p_3 \in \mathbb{R}^2$ each represent the positions of robots $1$, $2$ and $3$ respectively. 
The input space $\mathcal{U} = \mathbb{R}^6$ is the velocities of all robots where $\forall u \in \mathcal{U}, u = {\begin{bmatrix} v_1 & v_2 & v_3 \end{bmatrix}}^{\top}$ and $v_1, v_2, v_3 \in \mathbb{R}^2$ each represent the 
velocities of robots $1$, $2$ and $3$ respectively.\par
We first train the task of avoiding the region using an instantaneous cost of $q_1(x) + {\lVert u \rVert}^2$. 
$q_1(x) = a_1(x) + a_2(x) + a_3(x)$ where $a_i(x) = 35$ if robot $i$ is inside the $1$x$1$ region and $a_i(x) = 0$ otherwise.
We train this task as a single-robot task and covert it to a multi-robot task as described in Remark 3. 
We will denote the approximated cost functional for this first task as $\tilde{J}_{1}$. 

\begin{remark}
Consider a system comprising of $N$ robots where each robot has a state space $ \bar{\mathcal{X}} = \mathcal{R}^n$. 
The state space for the whole system is $\mathcal{X} = \mathcal{R}^{Nn}$. 
Consider a learned cost-to-go function of the form $\tilde{j}(\bar{x})$ where $\bar{x} \in \bar{\mathcal{X}}$. 
This can be turned into a multi-robot task by defining $\tilde{J}(x) = \sum_{i \in T} \tilde{j}(\bar{x}_i)$, for $x \in \mathcal{X}$ and $x_i \in \bar{\mathcal{X}}$, 
where $T$ is the set of robots to assign the single-robot task, $\tilde{j}$, to and $\bar{x}_i$ is the subset of state variables in $x$ used to represent the state of robot $i$.
When calculating $L_f \tilde{J}$ and $L_g \tilde{J}$ using back-propagation, every robot not assigned task $\tilde{j}$ will simply have gradients of zero.
This is similar to what is described in Remark $5$ in \cite{gennaroDars}.
\end{remark}\par
We then train the triangle formation task.
We train two trials using an instantaneous cost of $q_2(x) + {\lVert u \rVert}^2 $ using CFVI 
and two trials using an instantaneous cost of $q_2(x) + {\lVert u \rVert}^2 +  (L_g \tilde{J}_{1} (x) u )^2 \lambda$ to test our proposed method.
We set $\lambda = 50000$.
$q_2(x)$ is equal to:
$$
	(| \lVert p_1 - p_2 \rVert - 0.75 | + | \lVert p_1 - p_3 \rVert - 0.75 | + | \lVert p_2 - p_3 \rVert - 0.75 |) \cdot 1.5.
$$
$q_2(x) = 0$ when the robots form an equilateral triangle with each side of length $0.75$.\par
After training all versions of the formation task, we combine each version with $\tilde{J}_{1}$ and test on the same $50$ initial states. 
The initial states were randomly generated from a uniform distribution of a $2$x$2$ area encapsulating the region that $\tilde{J}_1$ was trained to avoid.
We considered the team of robots to have sucessfully completed both the avoidance and formation tasks if $q_1(x) = 0$ and $q_2(x) < 0.4$. 

\begin{table}[t]
	\caption{Success Rates of Avoiding Region and Forming a Triangle}
	\label{tab:locations}
	\begin{tabular}{cc}\toprule
		\textit{Training Method for Formation Task} & \textit{Success Rate} \\
		\midrule
		CFVI - Trial 1 & 0.4 \\
		CFVI - Trial 2 & 0.34 \\
		Ours - Trial 1 & 0.74 \\
		Ours - Trial 2 & 0.92 \\
	\end{tabular}
\end{table}

The results of this experiment are shown in Table 1. We can see that our proposed method allowed the team of robots to successfully avoid the region and form a triangle simultaneously much more frequently than when we simply
trained the formation task using CFVI. 

\begin{figure}
\includegraphics[width=0.6\linewidth]{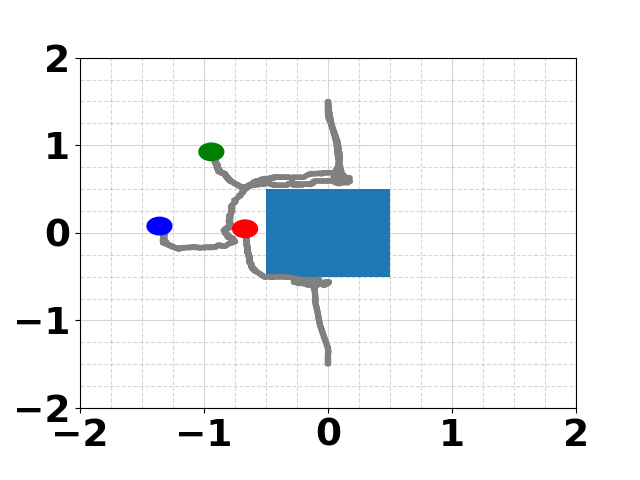}
\caption{Robot team forming triangle while avoiding region. Coloured Dots: Robots 1-3. Lighter Blue: Region to avoid. Grey: Trajectories of each robot.}
\Description{Instance of team of robots forming a triangle while avoiding the square region in the centre of the environment.}
\end{figure}

\subsection{Collaborative Transportation while Avoiding Square Regions}
We extend the previous example to train a team of three mobile robots to perform the following three tasks simultaneously: avoid three square-shaped regions in the environment, send one robot to a specified point and form the shape of a triangle. Like in the previous example, the state space, $\mathcal{X} = \mathbb{R}^6$, is the position of all robots where $\forall x \in \mathcal{X}, x = {\begin{bmatrix} p_1 & p_2 & p_3 \end{bmatrix}}^{\top}$
and $p_1, p_2, p_3 \in \mathbb{R}^2$ each represent the positions of robots $1$, $2$ and $3$ respectively. 
The input space $\mathcal{U} = \mathbb{R}^6$ is the velocities of all robots where $\forall u \in \mathcal{U}, u = {\begin{bmatrix} v_1 & v_2 & v_3 \end{bmatrix}}^{\top}$ and $v_1, v_2, v_3 \in \mathbb{R}^2$
each represent the velocities of robots $1$, $2$ and $3$ respectively. \par
We first train the task of avoiding three obstacles placed in the environment. We use an instantaneous cost function of
$q_1(x) + {\lVert u \rVert}^2$ where $q_1(x) = b_1(x) + b_2(x) + b_3(x)$ where $b_i(x) = 25$ if robot $i$ is inside any of the regions and $b_i(x) = 0$ otherwise. We will denote the learned cost-to-go function for this task as $\tilde{J}_{1}$.\par
We then train the triangle formation task. 
Similar to before, we use an instantaneous cost function of $q_2(x) + {\lVert u \rVert}^2$ for the baseline trials where we use CFVI 
and we use an instantaneous cost function of $q_2(x) + {\lVert u \rVert}^2 +  (L_g \tilde{J}_{1} (x) u )^2 \lambda$ to test our proposed method.
We set $\lambda = 50000$.
$q_2(x)$ is as defined previously in 5.3 where $q_2(x) = 0$ when the robots have formed an equilateral triangle of length 0.75. We train three baseline trials and three trials testing our proposed method.\par
Next, we train the single-robot task of going to the origin. We use an instantaneous cost of $q_3(\bar{x}) + {\lVert \bar{u} \rvert}^2$ for the  baseline task trained with CFVI and
an instantaneous cost of $q_3(\bar{x}) + {\lVert \bar{u} \rVert}^2 + ( L_g \tilde{J}_1 (\bar{x}) \bar{u} )^2$ for the task trained with our proposed method where $\bar{x} \in \mathcal{R}^2$ is the state for a single robot, $\bar{u} \in \mathbb{R}^2$ is the input velocity for a single robot and $q_3(\bar{x})$ is the
distance that robot is away from the origin multiplied by $12$. This single-robot task is turned into a multi-robot task in the way described in Remark $3$. Note that we did not train the go-to-point task to explicitly be independent to the formation task as it should already be independent by construction.\par
\begin{table}[t]
	\caption{Success Rates of Avoiding Region, Going to Origin and Forming a Triangle}
	\label{tab:locations}
	\begin{tabular}{ p{0.4\linewidth} | p{0.2\linewidth} | p{0.2\linewidth} }\toprule
		\textit{Training Method for Formation Task} & \textit{Training Method for Go-to-Point Task} & \textit{Success Rate} \\
		\midrule
		CFVI - Trial 1 & CFVI & 0.24 \\
		CFVI - Trial 1 & Ours & 0.18 \\

		CFVI - Trial 2 & CFVI & 0.16 \\
		CFVI - Trial 2 & Ours & 0.78 \\

		CFVI - Trial 3 & CFVI & 0.2 \\
		CFVI - Trial 3 & Ours & 0.7 \\

		Ours - Trial 1 & CFVI & 0.32 \\
		Ours - Trial 1 & Ours & 0.82 \\

		Ours - Trial 2 & CFVI & 0.18 \\
		Ours - Trial 2 & Ours & 0.54 \\

		Ours - Trial 3 & CFVI & 0.32 \\
		Ours - Trial 3 & Ours & 0.72 \\
	\end{tabular}
\end{table}

\begin{figure}[]
	\centering
	\begin{minipage}[c]{0.32\linewidth}
		\includegraphics[width=\linewidth]{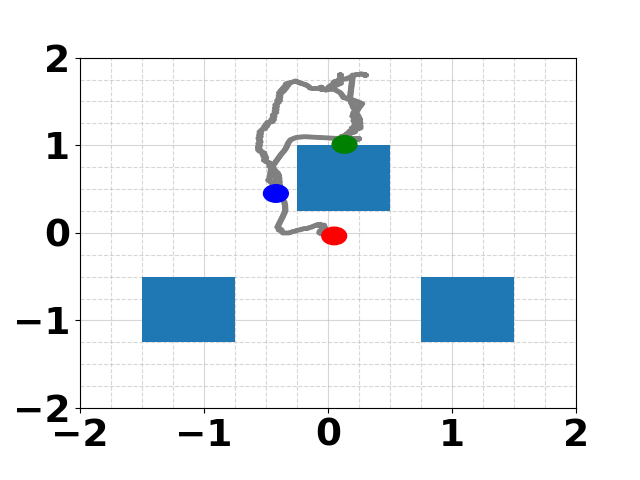}
	\end{minipage}
	\begin{minipage}[c]{0.32\linewidth}
		\includegraphics[width=\linewidth]{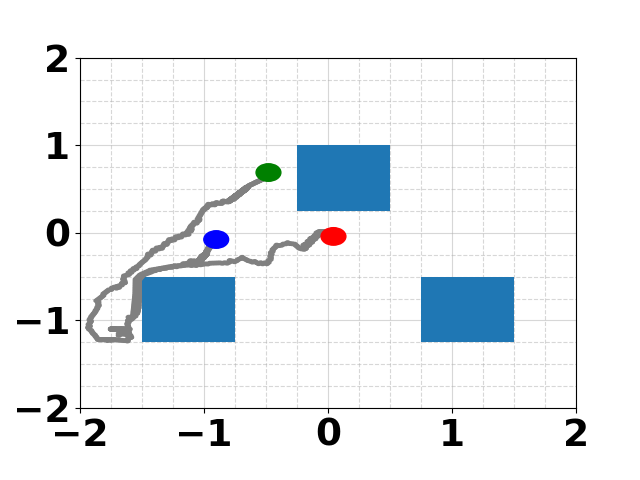}
	\end{minipage}
	\begin{minipage}[c]{0.32\linewidth}
		\includegraphics[width=\linewidth]{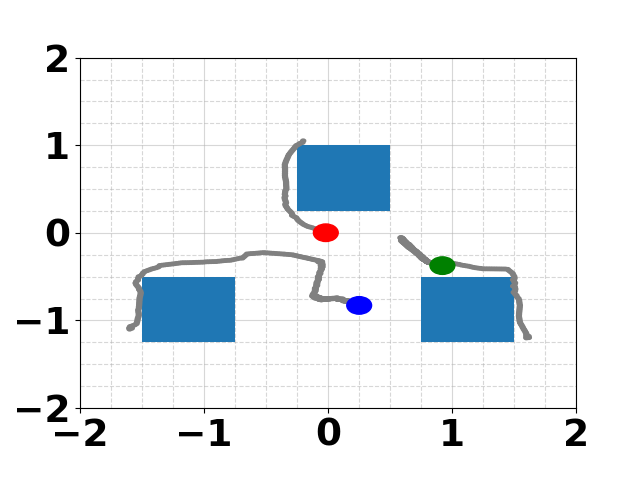}
	\end{minipage}%
	\caption{Robot team successfully forming triangle, sending first robot to origin and avoiding regions.
		 Coloured Dots: Robots 1-3. Lighter Blue: Regions to avoid. Grey: Trajectories of each robot.}
	\Description{Instances of team of robots successfully forming a triangle, sending the first robot to the origin and avoiding the three regions from different initial states.}
\end{figure}

\balance

Finally, we test each combination of trained models on $50$ randomly generated initial states and test the ability of each composed controller to achieve all three tasks. 
To randomly generate the initial states, we selected a position from a uniform distribution from space in the environment not covered by any of the three regions. We then added Gaussian noise to this position, to generate initial states
where each robot was approximately close to one another. This was done to simulate the team of robots transporting objects collaboratively.
We consider the system to have successfully completed all three tasks when $q_1(x) = 0$, when no robot is inside any of the three regions, $q_2(x) < 0.75$ and when $q_3(x) < 1.8$. The avoidance task was given the highest priority and the formation task was given the lowest. \par
The results of this experiment are shown in Table $3$. We can see that the best performing controller is composed of a formation task and go-to-point task trained using our proposed method  and that the worst
performing controller is one where the formation task and go-to-point task were both trained using just CFVI and not to be independent to the avoidance task.
The formation task seemed to benefit less from being trained to be independent to the avoidance task, compared to the previous example in 5.3, and this may be because of how the initial states were generated 
and because of the regions being smaller in this scenario. However, our proposed method still appeared to have made all three tasks more likely to be able to execute together concurrently.
\par
\subsubsection{Time-Varying Task Stacks}
This example also demonstrates the benefit of our overall approach to multi-objective RL. If we want to train our system to go to a different point or perform some other task while still avoiding the regions
and forming a triangle, we can simply train a new third task that is independent to these two tasks. 

\subsection{Training Implementation Details}
Note that the large values of $\lambda$ used in these simulations are not necessary and practitioners should be able to recreate similar results
using much smaller values of $\lambda$. Futhermore, values of $\lambda$ that are too large may lead to unstable training using fitted value iteration and other deep RL methods.
In some instances, practitioners may find it useful to limit the term that contains $\lambda$ in the instantaneous cost. To help further guide practitioners, we have published code here: \url{https://github.com/erablab/Value_Iteration_for_Learning_Concurrently_Executable_Robotic_Control_Tasks}.

\subsection{Experiment with Physical Robots}
We tested the best-performing controller from section 5.3 on physical robots with a physical square-shaped obstacle, the results of which can be viewed here: \url{https://youtu.be/8FSjqIr-pgY}. We used a Vicon motion capture system to track the robots for feedback and sent velocities using the API for the DJI RoboMaster EP Core robots. The dimensions of the physical region to avoid were adjusted to account for the size of the robots. 


\section{Conclusion}
In this paper, we present an approach for training multiple robotic control tasks in such a way that they can be combined and executed together in possibly time-varying prioritized stacks.
We start by defining a notion of independence between learned control tasks which characterize the abililty of tasks to be concurrently executed together as part of constraints in the same min-norm optimization-based controller.
We later propose a cost functional which can be fitted by RL algorithms such that the resulting learned tasks satisfy our definition of independence.
We also present a variant of the fitted value iteration algorithm that can be used to fit our cost functional, extending previous work related to value iteration in the continuous control setting.
Finally, we demonstrate our approach on several scenarios involving mobile robots, showing how our approach can be used to train tasks, that may have otherwise had conflicting objectives, to be concurrently executable.
  


\bibliographystyle{ACM-Reference-Format} 
\bibliography{root}


\end{document}